\documentclass{article}

\usepackage[arxiv]{icml2017}

\usepackage[utf8]{inputenc} 
\usepackage[T1]{fontenc}    
\usepackage{url}            
\usepackage{booktabs}       
\usepackage{nicefrac}       
\usepackage{microtype}      
\usepackage{times}
\usepackage{amsmath,amsfonts,amssymb,mathtools,mathrsfs,dsfont,bm}
\usepackage{amsthm}
\usepackage{graphicx}
\usepackage{subcaption}
\usepackage{wrapfig}
\usepackage{hyperref}
\usepackage[normalem]{ulem}
\usepackage{enumitem}
\usepackage{natbib}
\usepackage{dblfloatfix}
\usepackage{float}
\usepackage{multicol}
\usepackage{multirow}
\usepackage{array}
\usepackage{tabularx,colortbl}
\usepackage{hhline}
\usepackage{ragged2e}
\usepackage{tikz}

\definecolor{pale_blue}{rgb}{0.84,0.92,1}

\DeclareMathOperator*{\argmin}{arg\,min}
\DeclareMathOperator*{\argmax}{arg\,max}

\theoremstyle{definition}

\theoremstyle{theorem}
\newtheorem{theorem}{Theorem}
\theoremstyle{lemma}
\newtheorem{lemma}{Lemma}
\theoremstyle{proposition}
\newtheorem{proposition}{Proposition}
\theoremstyle{corollary}
\newtheorem{corollary}{Corollary}
\theoremstyle{lemma}

\makeatletter
\renewenvironment{proof}[1][\proofname]{\par
  \vspace{-\topsep}
  \pushQED{\qed}%
  \normalfont
  \topsep0pt \partopsep0pt 
  \trivlist
  \item[\hskip\labelsep
        \itshape
    #1\@addpunct{.}]\ignorespaces
}{%
  \popQED\endtrivlist\@endpefalse
  \addvspace{6pt plus 6pt} 
}
\makeatother

\allowdisplaybreaks

\begin{document}

\icmltitlerunning{Active Learning amidst Logical Knowledge}
\twocolumn[
\icmltitle{Active Learning amidst Logical Knowledge}
\begin{icmlauthorlist}
	\icmlauthor{Emmanouil A. Platanios}{cmu-msri}
	\icmlauthor{Ashish Kapoor}{msr}
	\icmlauthor{Eric Horvitz}{msr}
\end{icmlauthorlist}

\icmlaffiliation{cmu-msri}{Carnegie Mellon University, Pittsburgh, PA, USA (research performed during an internship at Microsoft Research)}
\icmlaffiliation{msr}{Microsoft Research, Redmond, WA, USA}

\icmlcorrespondingauthor{Emmanouil A. Platanios}{e.a.platanios@cs.cmu.edu}

\icmlkeywords{accuracy estimation, unsupervised learning, semi-supervised learning, never-ending learning, NELL}

\vskip 0.3in
]

\printAffiliationsAndNotice{}

\setlength{\belowdisplayskip}{5pt} \setlength{\belowdisplayshortskip}{5pt}
\setlength{\abovedisplayskip}{5pt} \setlength{\abovedisplayshortskip}{5pt}

\begin{abstract}

Structured prediction is ubiquitous in applications of machine learning such as knowledge extraction and natural language processing. Structure often can be formulated in terms of logical constraints. We consider the question of how to perform efficient active learning in the presence of logical constraints among variables inferred by different classifiers. We propose several methods and provide theoretical results that demonstrate the inappropriateness of employing uncertainty guided sampling, a commonly used active learning method. Furthermore, experiments on ten different datasets demonstrate that the methods significantly outperform alternatives in practice. The results are of practical significance in situations where labeled data is scarce.

\end{abstract}

\vspace{-0.6em}
\section{Introduction}
\vspace{-0.1em}

Tasks which involve learning several classifiers whose outputs are tied together by logical constraints are abundant in machine learning. As an example, we may have two classifiers in the Never Ending Language Learning (NELL) project \citep{Mitchell:2015wo} which predict whether noun phrases represent animals or cities, respectively. In this case, the outputs of the two classifiers are mutually exclusive. Many such tasks hinge on the training of a large number of classifiers in situations where obtaining labeled data is expensive. The difficulty of acquiring labels leads to the common approach (highlighted in Figure \ref{fig:example_active_learning_system}) of performing an initial training of classifiers with a small number of labeled examples, and then iteratively identifying the most valuable additional labels to acquire, followed by the re-training of the classifiers. We seek methods that are capable of performing such {\em active learning} \citep{Settles:2012eg}, an instance of {\em semi-supervised learning}. In this paper, we propose  methods for active learning that share a common underlying goal: efficient identification of the most valuable labels to acquire in the presence of {\em logical constraints} among the outputs of classifiers being trained. Examples of such constraints are mutual exclusion (e.g., in multi-class/one-vs-all classification) and subsumption (e.g., in hierarchical classification) among target variables. In active learning for mutual exclusion and subsumption, we need to consider the complexities of behavior arising in the interactions among the linked classifiers. We shall provide theoretical justification for the proposed methods that resonates with intuition. As we will show, the results challenge the core idea behind {\em uncertainty guided sampling}, a method in common practice.  

\begin{figure}[t!]
	\centering
    \includegraphics[width=0.5\textwidth,trim=110 70 50 80,clip]{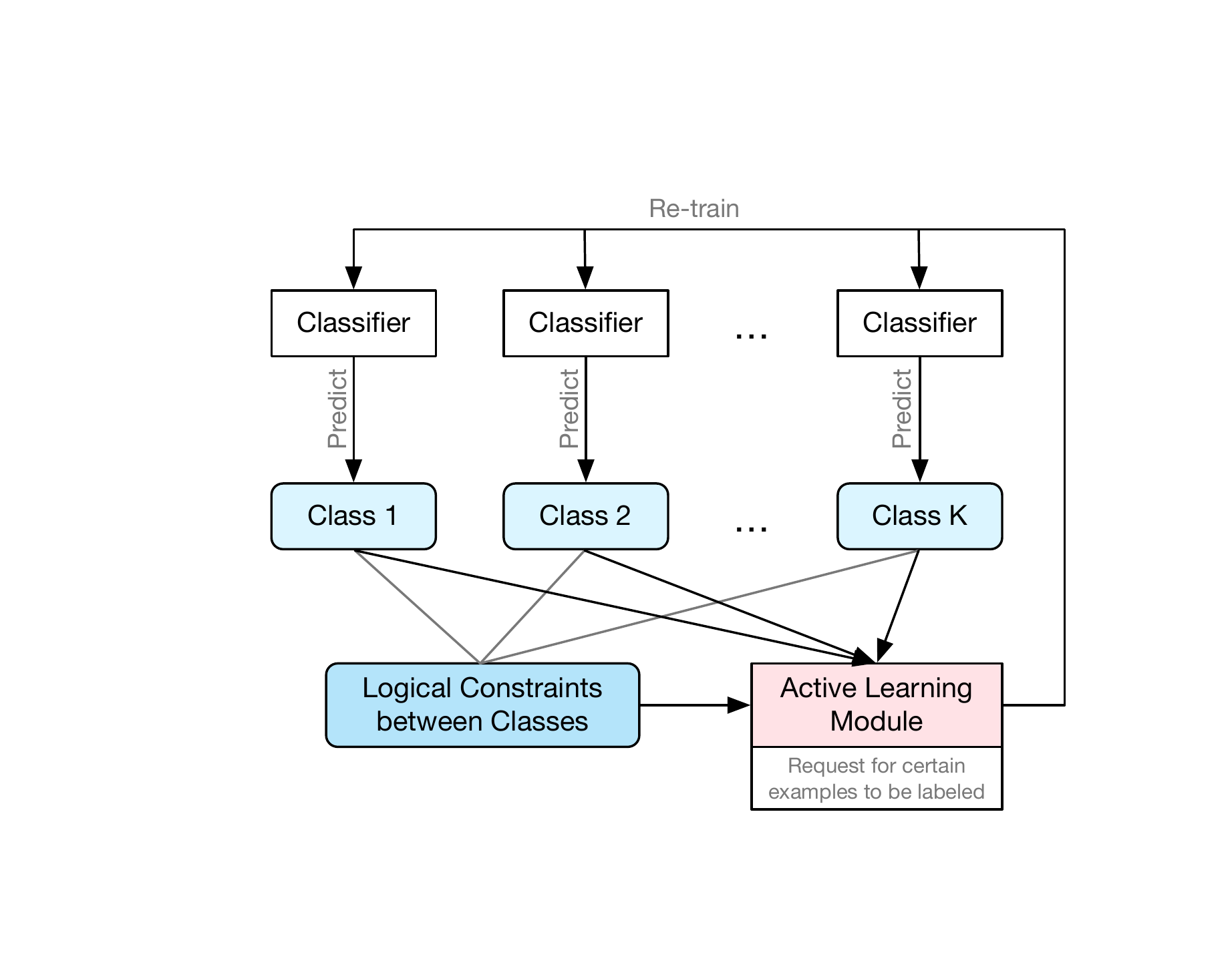}
    \vspace{-1.5em}
	\caption{Illustration of active learning in an interdependent multiple classifier setting.}
	\vspace{-1.0em}
	\label{fig:example_active_learning_system}
\end{figure}

We motivate our work with challenges in information extraction, where noun phrases are mapped to various categories (e.g., \texttt{animal}, and \texttt{bird}) and relations (e.g., \texttt{animalEatsFood}). It is easy to see how these categories and relations can be tied through logical constraints. For example, one might say that \texttt{animal} and \texttt{location} are mutually exclusive, and \texttt{animal} subsumes \texttt{bird}. We consider examples highlighted by work on the NELL project \citep{Mitchell:2015wo}. NELL currently performs over 2,500 learning tasks and it is thus too expensive to obtain enough labeled data for each task separately. The ability to rank examples by the utility of discovering their labels would enable the system to more efficiently allocate and use resources available for labeling. Our goal in this paper is to provide systems such as NELL with this ability, such that their learning rate would be significantly increased with respect to the resources available for labeling.

\vspace{-0.6em}
\section{Related Work}
\vspace{-0.15em}

The literature covers many projects in the realms of active learning \citep{Settles:2012eg,Ruvolo:2013wr} and decision theory (e.g., the core concept of {\em value of information} and its use in guiding the collection of examples) \citep{Kapoor:2007uu,Krause:2009wh}. Related work on computing the value of information for inference that leverages structural information includes an effort showing how the topology of influence diagrams could be used to assert an ordering over the value of information for variables \citep{Poh:1996uj}. However, most existing approaches to active information gathering for machine learning are directed at collecting single labels for one classifier. Furthermore, even approaches that deal with settings involving multiple labels do not make use of logical constraints that may exist among labels \citep{Reichart:2008wr,Zhao:2015ce}. Work in the area of semi-supervised learning makes it clear that such constraints are present in many practical settings and that they can indeed prove useful if used appropriately \citep{Chang:2007vqa,Chang:2008vh,Mitchell:2015wo}.

There have been a few approaches that make use of such constraints. First, we note that query-by-committee (QBC) can be viewed as a special case of our framework, where the logical constraint is that committee members must agree. \citet{Settles:2008:AAL:1613715.1613855} propose approaches to perform active learning for sequence labeling tasks, including uncertainty sampling and QBC. \citet{Culotta:2005:RLE:1619410.1619452} consider adding constraints to such tasks. To the best of our knowledge, they are the first to consider general constraints. In distinction to this prior work, we do not focus on the difficulty of each labeling task. We consider a wider range of tasks. Culotta and McCallum present only one instantiation of our more general formulation. \citet{Luo:2013lu} use uncertainty sampling, where probabilities are computed using classifiers that account for constraints. \citet{RothSm08} score instances using the margin of learned classifiers and in the case of binary classification, margin-based active learning is equivalent to uncertainty sampling. \citet{Bilgic:2010bi} consider dependencies among input instances and their labels. They cluster instances and look at disagreements of two kinds of classifiers over the clusters. Other methods that use ``side'' information in active learning include those of \citet{Kapoor:2009}, \citet{Wallace:2010wa}, and \citet{Angeli:2014an}. 

Our method considers the important and common case where there are logical constraints over the output space, such as mutual exclusion and subsumption. The ubiquitous nature of such logic relationships creates a need for them to be addressed “head on”. The previously mentioned related work only deals with other kinds of probabilistic constraints. \citet{Harpale:2012wq} and \citet{Zhang:2010ti} have considered this setting, but they fail to provide theoretical justifications nor provide deep experimental support. Furthermore both approaches can be seen as separate instantiations of our more general framework, for which we also provide a formal analysis along with an extensive experimental evaluation.

\vspace{-0.6em}
\section{Proposed Methods}
\label{sec:proposed_methods}
\vspace{-0.1em}

We now provide a description of methods for performing active learning. The methods select examples to be labeled before each re-training step (i.e., the red box in Figure \ref{fig:example_active_learning_system}). Let us consider a setting where we have a set of binary labels $\smash{Y_k^i\in\{0,1\}}$, for $k=1,\hdots,K$ and $i=1,\hdots,N$, for a provided set of instances $X^1,\hdots,X^N$. $\smash{Y_k^i}$ denotes whether instance $X^i$ belongs to class $k$. One example would be where $X^i$ represents a particular noun phrase (NP) and $Y_k^i$ is a particular label for that NP, indicating whether it is a city or not. There exists a set of logical constraints among the $K$ labels for each instance which determine whether an assignment of values to those labels is valid or not. Let the marginal probability of each label being positive be defined as $\smash{p_k^i\triangleq\mathbb{P}_{X^i\sim\mathcal{D}}(Y_k^i=1)}$, for $k=1,\hdots,K$ and $i=1,\hdots,N$, where $\mathcal{D}$ is the distribution of the instances $X^1,\hdots,X^N$. Given a set of observed labels (which could be empty) and these marginal probabilities, we want to determine which label\footnote{Note that the word ``label'' here refers to a particular label-instance pair (i.e., we ask for a single label of a single instance at a time). This is the convention we use throughout this paper.} to request in the active learning process in order to gain the most information. Thus, we use a {\em scoring function} to score each unobserved label based on how much information is gained by observing it, and we then pick the label with the highest such score. We note that {\em information gain} can be defined in many ways depending on the task at hand and the evaluation metric that is being used. Our approach is initially motivated by the loose and possibly na\"{i}ve definition of information gain as the expected number of labels one obtains after asking for a single label (i.e., due to the constraints among the labels).

We note that the common strategy of {\em uncertainty guided sampling} for allocating labeling effort uses the entropy of a label as its scoring function. That is:
\begin{equation}
\label{eq:entropy_scoring_function}
	\mathcal{S}_{\textrm{entropy}}(Y_k^i)\triangleq -p_k^i\log{p_k^i}-(1-p_k^i)\log(1-p_k^i).
\end{equation}
Thus, this function can be thought of as scoring each label based only on its own uncertainty ignoring any dependencies among the labels. Our proposed methods make use of logical constraints among the labels, thus considering key dependencies. We therefore expect them to perform better in practice.

\subsection{A Simple Constraint: Mutual Exclusion}
\label{sec:mutual_exclusion}

Let us first consider a simple, yet powerful and common logical constraint among labels: {\em mutual exclusion}. We consider a setting where, for each value of $i$ (i.e., instance), all labels (i.e., $Y_1^i,\hdots,Y_K^i$) are mutually exclusive with one another. This means that, for each instance, at most one label can be positive. It is easy to see that, if we discover that a label for a specific instance is positive, then all other labels must be negative for that instance. However, if the answer is negative, then we cannot infer the value of any other label. Thus, intuitively we see that it might make sense to ask for the label with the highest marginal probability of being equal to $1$ (i.e., the $Y_k^i$ with the highest probability $p_k^i$). We now discuss this approach and provide theoretical justification for this intuition. We start by suggesting the following scoring function:
\begin{equation}
\label{eq:probability_scoring_function}
	\mathcal{S}_{\textrm{probability}}(Y_k^i)\triangleq p_k^i.
\end{equation}
For the following theoretical justification, we shall ignore the instance superscript (i.e., $i$) and consider the case where there is only a single instance $X$.  We shall propose a theorem related to this scoring function, but first we will state a lemma that will be used in the forthcoming proof:
\begin{lemma}
\label{lem:mutual_exclusion_theorem_lemma}
Let $x\in[0,1]$, and $c\in[0,1-x]$. Then, the following function is monotonic with respect to $x$: $f(x)=(1-x-c)\log{(1-x-c)}-(1-x)\log{(1-x)}$.
\end{lemma}
\begin{proof}
We have that $\frac{\partial f(x)}{\partial x}=\log{(1-x)}-\log{(1-x-c)}$ and since the logarithm is a monotonic function, we know that $\frac{\partial f(x)}{\partial x}\geq 0$. Thus, $f(x)$ is monotonic.
\end{proof}

\begin{theorem}
\label{thm:mutual_exclusion_theorem}
Given a set of mutually exclusive labels, the scoring function in Equation \ref{eq:probability_scoring_function} induces the same ranking of labels as the information-theoretic information gain.
\end{theorem}
\begin{proof}
Due to the mutual exclusion constraint, we have $\mathbb{P}_{X\sim\mathcal{D}}(\{Y_k=0\textrm{ for }k=1,\hdots,K\})=1-\smash{\sum_{k=1}^K{p_k}}$. For notational convenience, let us denote this quantity by $p_0$ and also omit the $X\sim\mathcal{D}$ subscript from the probability operator notation henceforth. Now, note that:
\begin{equation*}
\begin{split}
	\mathbb{P}(\bm{y_{-k}})&=\mathbb{P}(\bm{y_{-k}}\land Y_k=1)+\mathbb{P}(\bm{y_{-k}}\land Y_k=0),\\
	&=\begin{cases}
		0 & \textrm{, if }\bm{y_{-k}}\textrm{ has more than one 1s}, \\
		p_l & \textrm{, if }y_l=1\textrm{ for }l\neq k, \\
		p_k+p_0 & \textrm{, otherwise},
	\end{cases}
\end{split}
\end{equation*}
where $\bm{y_{-k}}$ refers to an assignment of values to all labels $Y_l$, where $l=1,\hdots,K$, and $l\neq k$, and $y_l$ refers to an assignment of $Y_l$. Let us also denote the information-theoretic information gain of variable $Y_k$ by $\mathcal{I}(Y_k)$. We then have that if $p_k\geq p_l$, for some $k\neq l$, then:
\begin{align*}
	&\mathcal{I}(Y_k)-\mathcal{I}(Y_l) \\
	&\;=\mathcal{H}(\bm{Y_{-k}})-\mathcal{H}(\bm{Y_{-k}}\mid Y_k)-\mathcal{H}(\bm{Y_{-l}})+\mathcal{H}(\bm{Y_{-l}}\mid Y_l), \\
	&\;=\mathcal{H}(\bm{Y_{-k}})+\mathcal{H}(Y_k)-\mathcal{H}(\bm{Y_{-l}})-\mathcal{H}(Y_l), \\
	&\;=-\sum_{\bm{y_{-k}}}{\mathbb{P}(\bm{y_{-k}})\log{\mathbb{P}(\bm{y_{-k}})}}-p_k\log{p_k}, \\
	&\hspace{0.49cm}+\sum_{\bm{y_{-l}}}{\mathbb{P}(\bm{y_{-l}})\log{\mathbb{P}(\bm{y_{-l}})}}+p_l\log{p_l}, \\
	&\hspace{0.49cm}-(1-p_k)\log{(1-p_k)}+(1-p_l)\log{(1-p_l)}, \\	
	&\;=(p_l+p_0)\log{(p_l+p_0)}-(1-p_k)\log{(1-p_k)}, \\
	&\hspace{0.49cm}-(p_k+p_0)\log{(p_k+p_0)}+(1-p_l)\log{(1-p_l)}, \\
	&\;=(1-p_k-c)\log{(1-p_k-c)}-(1-p_k)\log{(1-p_k)}, \\
	&\hspace{0.49cm}-(1-p_l-c)\log{(1-p_l-c)}+(1-p_l)\log{(1-p_l)}, \\
	&\;\geq 0,
\end{align*}
where $\mathcal{H}(Y_k)$ corresponds to the entropy of the $Y_k$ variable, $\mathcal{H}(\bm{Y_{-k}})$ corresponds to the entropy of all variables $Y_l$, where $l=1,\hdots,K$ and $l\neq k$, and $\smash{H(p)\triangleq-p\log{p}-(1-p)\log{(1-p)}}$. The sums are over all possible assignments of the corresponding variables. The last step follows from Lemma \ref{lem:mutual_exclusion_theorem_lemma}, where $\smash{c=\sum_{k'=1,k'\neq k,l}^K{p_{k'}}}$. The above inequality implies that the ranking of labels induced by the information gain $\mathcal{I}(Y_k)$ is the same as the ranking induced by using the scoring function $\smash{\mathcal{S}_{\textrm{probability}}(Y_k^i)}$ and the proof is complete.
\end{proof}

One of the most interesting consequences of Theorem \ref{thm:mutual_exclusion_theorem} is that we now have a very efficient way to rank labels based on their information gain. Also note that more often than not, classification systems are evaluated based on the area under the precision-recall curve (AUC). Furthermore, one might care about maximizing the number of ``gold'' labels, meaning labels that are guaranteed to be correct. Intuitively, the AUC increases with the number of ``gold'' labels. We highlight the fact that the probability scoring function of Equation \ref{eq:probability_scoring_function} was motivated by picking the label that is most likely to provide the greatest number of ``gold'' labels (i.e., labels that are fixed to $0$).

We note that the scoring function $\mathcal{S}_{\textrm{probability}}(Y_k^i)$ assigns higher score to labels that are more certainly positive than to more uncertain labels. This is in contrast to uncertainty guided sampling and thus highlights the importance of Theorem \ref{thm:mutual_exclusion_theorem}. It also demonstrates that positive examples can, in some cases, be much more useful and informative than negative examples. \citet{Sharma:2013sh} discuss the sources of uncertainty and reinforce our argument about the ineffectiveness of naive uncertainty sampling for some kinds of logical constraints.

We would like to emphasize the relationship between using Equation \ref{eq:probability_scoring_function} as the scoring function, as proposed in Theorem \ref{thm:mutual_exclusion_theorem}, and using entropy (i.e., Equation \ref{eq:entropy_scoring_function}) as the scoring function, as is done in uncertainty guided sampling. The following proposition and corollary of our theorem describe this relationship more precisely.


\begin{proposition}
\label{prop:entropy_cp_equivalence_condition}
When:
\begin{equation}
	\argmax_{\substack{k=1,\hdots,K,\\i=1,\hdots,N}}{p_k^i}=\argmin_{\substack{k=1,\hdots,K,\\i=1,\hdots,N}}{|p_k^i-0.5|},
\end{equation}
the probability scoring function of Equation \ref{eq:probability_scoring_function} is equivalent to the entropy scoring function of Equation \ref{eq:entropy_scoring_function}, which is used by uncertainty guided sampling.
\end{proposition}
\begin{proof}
The proof follows immediately by noticing that for $p_k^i\in[0,1]$, the following holds:
\begin{equation*}
	\argmax_{\substack{k=1,\hdots,K,\\i=1,\hdots,N}}{\mathcal{S}_{\textrm{entropy}}(Y_k^i)}\equiv\argmin_{\substack{k=1,\hdots,K,\\i=1,\hdots,N}}{|p_k^i-0.5|}.\qedhere
\end{equation*}
\end{proof}
\begin{corollary}
\label{cor:entropy_cp_single_instance_special_case}
In the case of a single instance $X$ the probability scoring function of Equation \ref{eq:probability_scoring_function} and the entropy scoring function of Equation \ref{eq:entropy_scoring_function} are equivalent.
\end{corollary}
\begin{proof}
Due to the mutual exclusion constraint, $\smash{\sum_{k=1}^K{p_k}\leq 1}$, which implies that the condition of Proposition \ref{prop:entropy_cp_equivalence_condition} is always satisfied.
\end{proof}
It is easy to observe that when we have several instances $X^1,\hdots,X^N$ and we compare the scores of each label-instance pair, then the two scoring functions are no longer necessarily equivalent. Also note that as the number of labels grows, the marginals are more likely to have smaller magnitudes and thus the condition of Proposition \ref{prop:entropy_cp_equivalence_condition} is more likely to be satisfied.

\paragraph{A Different Approach.}

We now introduce a new concept that, when combined with the earlier motivation, gives rise to a new scoring function. The key intuition lies in the scenario where the discovery of a label being positive implies that all other labels are negative. This discovery may not be as valuable if the negative labels were already inferred to have low marginal probability. Instead, we propose to consider the {\em degree of surprise} of discovering that those labels are negative. For a label with marginal probability of being positive $p_k$, the amount of surprise can be defined in several ways. A function $\mathscr{S}:[0,1]\mapsto\mathbb{R}$ is called a {\em surprise function} if it is decreasing and $\mathscr{S}(1) = 0$. A couple examples of such {\em surprise functions} are shown here:
\begin{itemize}[noitemsep,topsep=0pt,leftmargin=*]
	\item {\sc Logarithmic}:	 $\mathscr{S}_{\textrm{log}}(p_k)\triangleq -\log{p_k}$. This is equivalent to the {\em self-information} of the event $Y_k=1$ and was first referred to as a surprise measure by \citet{Tribus:1961}.
	\item {\sc Linear}: $\mathscr{S}_{\textrm{lin}}(p_k)\triangleq 1-p_k$. This is associated to the 0-1 loss of \citet{Roy:2001wn}.
\end{itemize}
Using this definition of a surprise function, we define a new scoring function for the mutual exclusion case as follows:
\begin{equation}
\label{eq:mutual_exclusion_scoring_function}
\begin{split}
	\mathcal{S}_{\textrm{ME}}(Y_k)&\triangleq p_k\underbrace{\sum_{c=1}^K{\mathds{1}_{c=k}\mathscr{S}(p_c)+\mathds{1}_{c\neq k}\mathscr{S}(1-p_c)}}_{\textrm{Total surprise of setting }Y_k=1}, \\
	&\quad +(1-p_k)\underbrace{\vphantom{\sum_{c=1}^K}\mathscr{S}(1-p_k)}_{\mathclap{\textrm{Total surprise of setting }Y_k=0}},
\end{split}
\end{equation}

where $\mathscr{S}(\cdot)$ is an arbitrary surprise function, and $\mathds{1}_{\cdot}$ is the indicator function which is equal to $1$ if the condition in the subscript is satisfied and is equal to $0$ otherwise. Note that the first term is the product of the probability of $Y_k$ being equal to $1$ and the sum of surprise ``experienced'' by fixing the value of $Y_k$ to $1$ (i.e., after propagating the mutual exclusion constraint, we sum over the surprises of all other labels being set to $0$ and $Y_k$ being set to $1$). The second term is similarly defined as the product of the probability of $Y_k$ being equal to $0$ and the surprise of fixing $Y_k$ to that value. No other variables are considered in this surprise value, as no other label value is fixed because of the mutual exclusion constraint. Note that this is substantially different than the entropy scoring function in that it's measuring ``surprise'' rather than uncertainty.

\vspace{-0.5em}
\subsection{More General Logical Constraints}
\vspace{-0.1em}

The scoring function of Equation \ref{eq:mutual_exclusion_scoring_function} and the underlying intuition can easily be extended to more general logical constraints than mutual exclusion. An example of a more general logical constraint is {\em subsumption}. In this case, each label can have a set of parent and child labels, and a label being set to $1$ implies that its parent label is $1$. To extend the method introduced in the previous section, we need a function for propagating a fixed label-value pair through the constraints. Let this function be defined as $\smash{\mathcal{F}(Y_k = v) \triangleq \left\{ (Y_{c_i}, v_i) : \text{if } Y_k = v \text{, then } Y_{c_i} = v_i \right\}}$, where $c_i\in\{1,\hdots,K\}$ is a label index, and $v_i\in\{0,1\}$ is the value of $Y_{c_i}$ fixed by propagating the fixed label-value pair $(Y_k, v)$ through the constraints. We can now define our scoring function for general logical constraints as follows:
\begin{equation}
\label{eq:constraints_scoring_function}
\begin{split}
	&\mathcal{S}_{\textrm{constraints}}(Y_k)\triangleq p_k\underbrace{\sum_{(Y_{c_i},v_i)\in\mathcal{F}(Y_k = 1)}{S(Y_{c_i},v_i)}}_{\textrm{Total surprise of setting }Y_k=1}, \\
	&\hspace{1.9cm}+(1-p_k)\underbrace{\sum_{(Y_{c_i},v_i)\in\mathcal{F}(Y_k = 0)}{S(Y_{c_i},v_i)}}_{\textrm{Total surprise of setting }Y_k=0},
\end{split}
\end{equation}

where $S(Y_{c_i},v_i)=\mathds{1}_{v_i=1}\mathscr{S}(p_{c_i})+\mathds{1}_{v_i=0}\mathscr{S}(1-p_{c_i})$.

\vspace{-0.4em}
\paragraph{Formal Justification.}

We have not derived a result for the general scoring function similar to that of Theorem \ref{thm:mutual_exclusion_theorem}. However, we can use the information-theoretic information gain to generate an interesting result, akin to our justification for using the scoring function of Equation \ref{eq:probability_scoring_function} in the setting of mutual exclusion. We note that the information gain for the case with general logical constraints can be defined as a sum. The first term of this sum is the entropy of the label whose information gain is being computed. When the logarithmic surprise function is used with the scoring function of Equation \ref{eq:constraints_scoring_function}, then our scoring function contains this entropy term, as well as an approximation of some other terms (but not all) of the complete information gain sum. More specifically, we have that (in this derivation we ignore terms that are constant across all label variables, since these terms do not affect the ranking of the labels induced by the information gain):

\begin{align*}
	&\mathcal{I}(Y_k)=\mathcal{H}(Y_{-k})-\mathcal{H}(Y_{-k}\mid\mathcal{H}_k), \\
	&\;=\mathcal{H}(Y_{-k})+\mathcal{H}(Y_k)-H(Y)=\mathcal{H}(Y_k)-\mathcal{H}(Y_k\mid\bm{Y_{-k}}), \\
	&\;=\sum_{y_k}\bigg[-\mathbb{P}(y_k)\log{\mathbb{P}(y_k)}, \\[-7pt]
	&\hspace{1.5cm}+\sum_{\bm{y_{-k}}}{\mathbb{P}(\bm{y_{-k}})\mathbb{P}(y_k\mid\bm{y_{-k}})\log{\mathbb{P}(y_k\mid\bm{y_{-k}})}}\bigg], \\
	&\;=\sum_{y_k}\bigg[-\mathbb{P}(y_k)\log{\mathbb{P}(y_k)}, \\[-7pt]
	&\hspace{1.5cm}+\sum_{\bm{y_{-k}}}\big[\underbrace{\mathbb{P}(y_k,\bm{y_{-k}})\log{\mathbb{P}(y_k,\bm{y_{-k}})}}_{\textrm{Constant}}, \\[-10pt]
	&\hspace{3.0cm}-\mathbb{P}(y_k,\bm{y_{-k}})\log{\mathbb{P}(\bm{y_{-k}})\big]}\bigg], \\
	&\;=\sum_{y_k}\mathbb{P}(y_k)\bigg[-\log{\mathbb{P}(y_k)}, \\[-7pt]
	&\hspace{1.5cm}-\sum_{\mathclap{\substack{\bm{y_{-f}}\\\textrm{where }\bm{y_f}=\mathcal{F}(y_k)}}}{\quad\underbrace{\mathbb{P}(\bm{y_{-f}},\bm{y}_{\bm{f} \setminus k}\mid y_k)}_{\mathbb{P}(\bm{y_{-f}}\mid\bm{y_f})}\log{\mathbb{P}(\bm{y_{-f}},\bm{y}_{\bm{f} \setminus k})}}\bigg], \\
	&\;=\sum_{y_k}\mathbb{P}(y_k)\bigg[-\log{\mathbb{P}(y_k)}, \\[-7pt]
	&\hspace{1.5cm}-\sum_{\mathclap{\substack{\bm{y_{-f}}\\\textrm{where }\bm{y_f}=\mathcal{F}(y_k)}}}\quad\big[\underbrace{\mathbb{P}(\bm{y_{-f}}\mid\bm{y_f})}_{\textrm{Sums to }1}\log{\mathbb{P}(\bm{y_f})}, \\[-7pt]
	&\hspace{3.0cm}+\mathbb{P}(\bm{y_{-f}}\mid\bm{y_f})\log{\mathbb{P}(\bm{y_{-f}}\mid\bm{y}_{\bm{f} \setminus k})}\big]\bigg], \\
	&\;=\sum_{y_k}\mathbb{P}(y_k)\bigg[-\underbrace{\log{\mathbb{P}(y_k)}}_{\textrm{Entropy}}-\underbrace{\log{\mathbb{P}(\bm{y_f})}}_{\textrm{Constraints}}, \\[-7pt]
	&\hspace{2.0cm}-\sum_{\mathclap{\substack{\bm{y_{-f}}\\\textrm{where }\bm{y_f}=\mathcal{F}(y_k)}}}{\quad\underbrace{\mathbb{P}(\bm{y_{-f}}\mid\bm{y_f})\log{\mathbb{P}(\bm{y_{-f}}\mid\bm{y}_{\bm{f} \setminus k})}}_{\textrm{Remainder}}}\bigg],
\end{align*}

where $\bm{f}\setminus k$ is the set of label indices in $\bm{f}$ excluding $k$. Note that the entropy scoring function of Equation \ref{eq:entropy_scoring_function} only considers the term denoted by ``Entropy'' in this sum. When the logarithmic surprise function is used, the general scoring function of Equation \ref{eq:constraints_scoring_function} contains an approximation to the terms denoted by ``Constraints'' where the joint is written as the product of the marginals. This result shows that the general scoring function is, in some sense, a better heuristic for the information gain than the entropy scoring function.  

\vspace{-0.3em}
\subsection{Computational Complexity}
\vspace{-0.1em}

We will now consider the real-world use of an active learning system, where a request is made for a label of a particular instance. Note that if we were to use the information-theoretic information gain as our scoring function, then the cost would be linear in $N$ and exponential in $K$. Our scoring functions reduce this cost. The entropy scoring function of Equation \ref{eq:entropy_scoring_function} has a computational cost linear in the number of labels and the number of instances (i.e., because we need to compute it for all labels); its cost is $O(NK)$. The probability scoring function of Equation \ref{eq:probability_scoring_function} has the same cost. The mutual exclusion scoring function has cost $O(NK^2)$. Finally, ignoring the cost of the constraint propagation function, the general scoring function of Equation \ref{eq:constraints_scoring_function} has a computational cost of $O(NK^2)$, since the highest number of labels that can be fixed is $K$. Note that the constraint propagation function can have a cost exponential in $K$ in the worst case. However, there are special cases where the cost of that operator is not as high. For example, with either the mutual exclusion or the subsumption constraint the cost is linear in $K$. When mutual exclusion is combined with subsumption, we can alternate between all our constraints, one by one, and keep propagating them, until no fixed label can be propagated further. If the number of constraints is $C$, the cost of that propagation operation is $O(CK)$. This is the most complex scenario that we consider in our experiments and covers most of the practical use cases in multi-task applications.

%

\vspace{-0.6em}
\section{Experiments}
\label{sec:experiments}
\vspace{-0.1em}

In the following paragraphs, we describe the setup of our experiments, including the datasets and the evaluation metrics that we use, and the results and their corresponding analyses. All the datasets and code for the experiments are available at \url{https://github.com/eaplatanios/makina}.

\begin{table*}[!t]
\vspace{0.1em}
\caption{Datasets used in experiments.}
\vspace{-0.5em}
\label{tab:data_sets}
\vspace{-0.1in}
\begin{center}
\begin{small}
\begin{sc}
\begin{tabular*}{\textwidth}{l@{\extracolsep{\fill}}r@{\extracolsep{\fill}}r@{\extracolsep{\fill}}c@{\extracolsep{\fill}}r@{\extracolsep{\fill}}r@{\extracolsep{\fill}}r}
\hline
Dataset	& \#Classes	& \#Features		& Balanced?	& \#Training & \#Testing	& \#Requested/Iteration \\
\hline
SatImage	& 6			& 36				& $\times$	& 3,104		& 1,331			& 100 \\
Shuttle		& 7			& 9				& $\times$	& 30,450		& 13,050		& 1,000 \\
Segment		& 7			& 19				& $\surd$	& 400		& 1,910			& 100 \\
PenDigits	& 10		& 16				& $\surd$	& 7,494		& 3,498			& 100 \\
Letter		& 26		& 16				& $\surd$	& 15,000		& 5,000			& 1,000 \\
Nell-7		& 7			& 180,878		& $\times$	& 214		& 14,693		& 500 \\
Nell-11		& 11		& 180,878		& $\times$	& 242		& 14,693		& 1,000 \\
Nell-13		& 13		& 180,878		& $\times$	& 2,656		& 18,016		& 2,000 \\
\hline
\end{tabular*}
\end{sc}
\end{small}
\end{center}
\vskip -0.05in
\end{table*}

\begin{figure*}[!t]
	\centering
	\vspace{-0.3cm}
	\includegraphics[width=\textwidth]{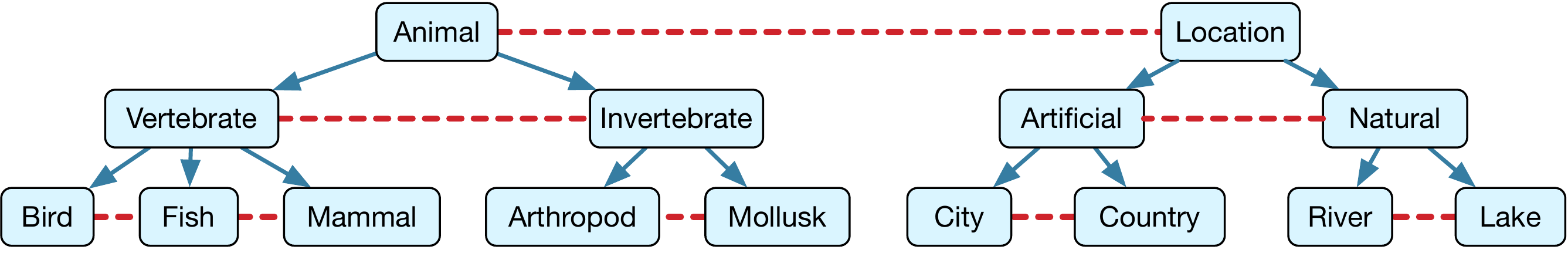}
	\vspace{-0.7cm}
	\caption{Illustration of the {\sc Nell-13} dataset constraints. Each box represents a label, each blue arrow represents a subsumption constraint, and each set of labels connected by a red dashed line represents a mutually exclusive set of labels.} 
	\label{fig:NELL_13_constraints}
\end{figure*}

We first define the names that we use to refer to different methods when plotting the results:
\begin{itemize}[noitemsep,topsep=0pt,leftmargin=*]
	\item \texttt{RANDOM} uses a random scoring function (i.e., using a random between $0$ and $1$ for the score).
	\item \texttt{ENTROPY} uses  the entropy scoring function of Equation \ref{eq:entropy_scoring_function}.
	\item \texttt{RANDOM-CP} is the same as \texttt{RANDOM}, but also propagates labels through the constraints.
	\item \texttt{ENTROPY-CP} is the same as \texttt{ENTROPY}, but also propagates labels through the constraints.
	\item \texttt{PROBABILITY-CP}: Using the probability scoring function of Equation \ref{eq:probability_scoring_function} and also propagating labels through the constraints.
	\item \texttt{LOG-CP}: Using the constraints scoring function of Equation \ref{eq:constraints_scoring_function} with the logarithm surprise function and also propagating labels through the constraints.
	\item \texttt{LINEAR-CP}: Same as \texttt{LOG-CP}, but using the linear surprise function instead of the logarithm.
\end{itemize}
We apply the same experimental setup to all of the datasets. Each dataset consists of a set of positive examples for each label. For each experiment, we split the dataset into training and testing subsets. For each label, we train a binary logistic regression classifier using the AdaGrad stochastic optimization algorithm of \citep{Duchi:2011wu}, with a batch size of $100$ samples per iteration. Our experimental pipeline consists of the following steps:
\begin{enumerate}[noitemsep,topsep=0pt,leftmargin=*]
	\item We initially train a classifier for each label independently using the training portion of the dataset. We consider all positive examples for the corresponding label, along with a set of negative examples of the same size, sampled from the remaining set of examples in the training dataset. 
	\item We repeat the following steps until all of the testing data have been manually labeled:
	\begin{enumerate}[noitemsep,topsep=0pt]
		\item Request a set of $M$ examples from the testing dataset to be manually labeled\footnote{Note that by ``example'' we mean a label-instance pair and so all possible label instance pairs from the testing dataset are considered at this stage.}, sequentially. For all the methods that include \texttt{CP} in their name, after each example is obtained, propagate all the logical constraints. The examples fixed by this process are considered manually labeled. Note that $M$ can vary across each dataset since they differ in size. Please refer to Table \ref{tab:data_sets} for the values of $M$ used for each dataset. Also note that this step differs across our methods. Each method's scoring function determines which examples are selected for labeling. The label-instance pair with the highest score is selected for labeling.
		\item Move all the labeled examples from the testing to the training dataset.
		\item Re-train the classifiers for all the labels, using the updated training dataset. Training for the classifiers is initialized at the previously learned point to reduce convergence time.
		\item Evaluate progress using a set of metrics and the full dataset (i.e., the training and testing parts of the dataset, combined). Note that even though it may seem unorthodox to evaluate on the full dataset, it is actually meaningful for settings like NELL. In fact, that is how NELL is evaluated, as we care about the accuracy of its whole knowledge-base, irrespective of how the label of an instance was obtained.
	\end{enumerate}
\end{enumerate}
\paragraph{A Note on Marginal Probabilities.} Note that all our methods and results of section \ref{sec:proposed_methods} rely on marginal probabilities. In our experiments, we use classifiers to estimate those marginals and sometimes they may not be very accurate. This is actually the reason we subsample a number of negative examples equal to the number of positive examples. Otherwise, our logistic regression classifiers would be biased towards low estimates of the probabilities, which would cause the entropy and our proposed scoring functions to perform very similarly, as shown in section \ref{sec:mutual_exclusion}. This was indeed the case when we ran experiments without subsampling the negative examples. This problem can also be alleviated by using more appropriate classifiers for the problem, than logistic regression.

\begin{figure*}[t!]
\scriptsize
    \begin{tabular}{m{0.08\textwidth} m{0.92\textwidth}}
    {\sc SatImage} & \includegraphics[width=0.92\textwidth,trim=0 60 0 0,clip]{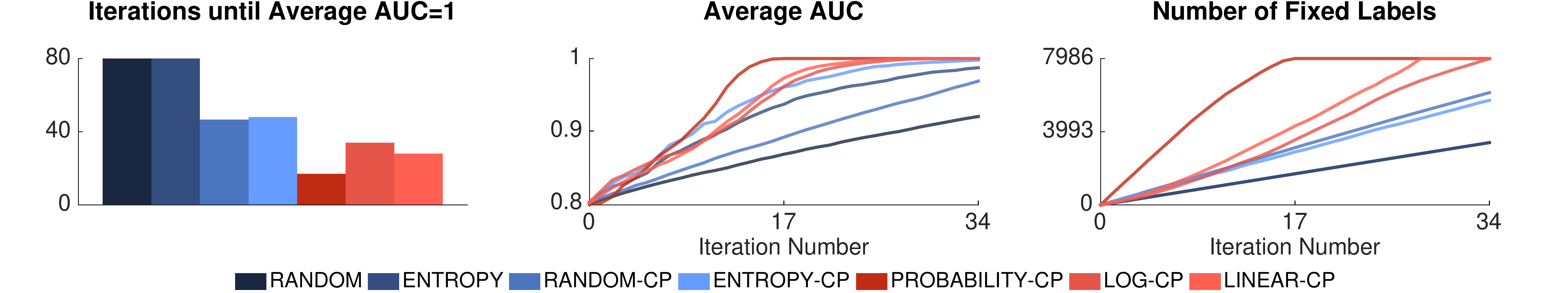} \\
    {\sc Shuttle} & \includegraphics[width=0.92\textwidth,trim=0 60 0 28,clip]{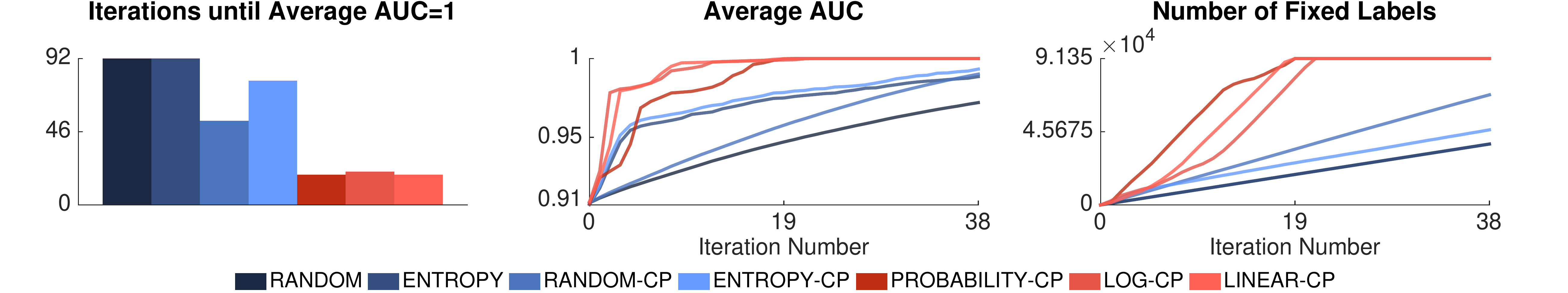} \\
    {\sc Segment} & \includegraphics[width=0.92\textwidth,trim=0 60 0 28,clip]{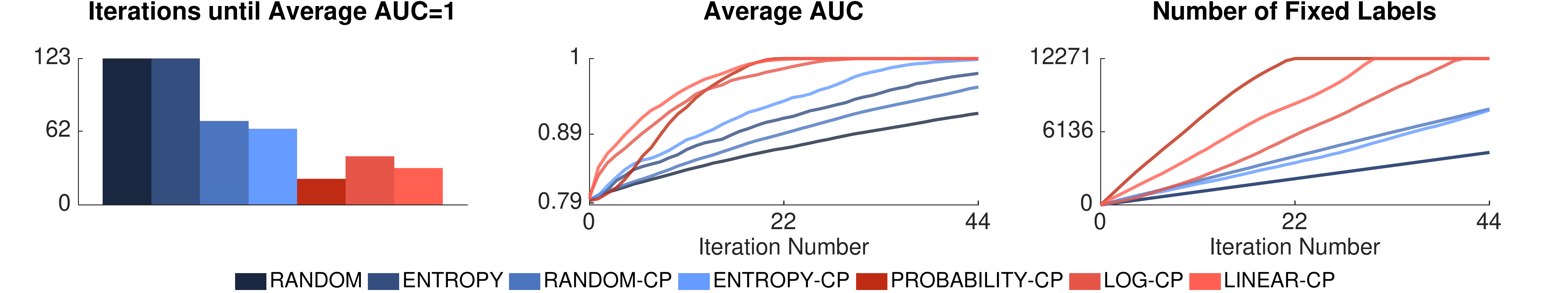} \\
    {\sc PenDigits} & \includegraphics[width=0.92\textwidth,trim=0 60 0 28,clip]{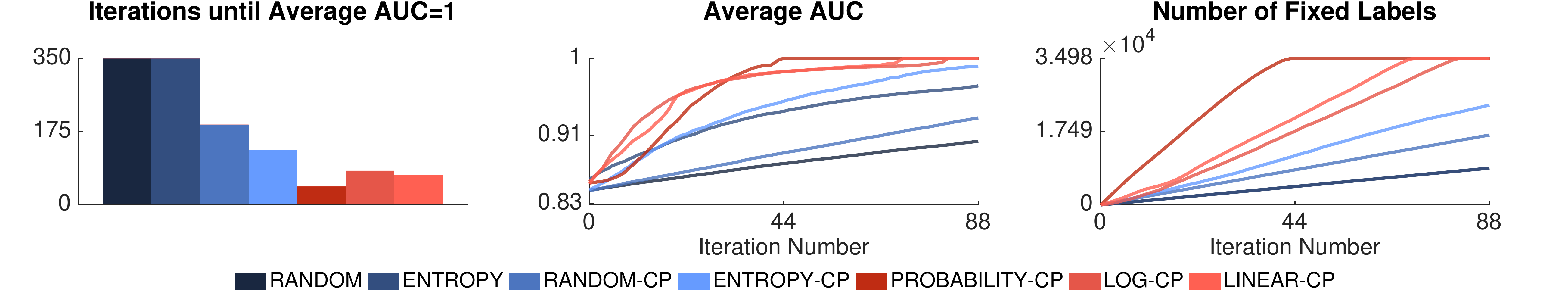} \\
    {\sc Letter} & \includegraphics[width=0.92\textwidth,trim=0 60 0 28,clip]{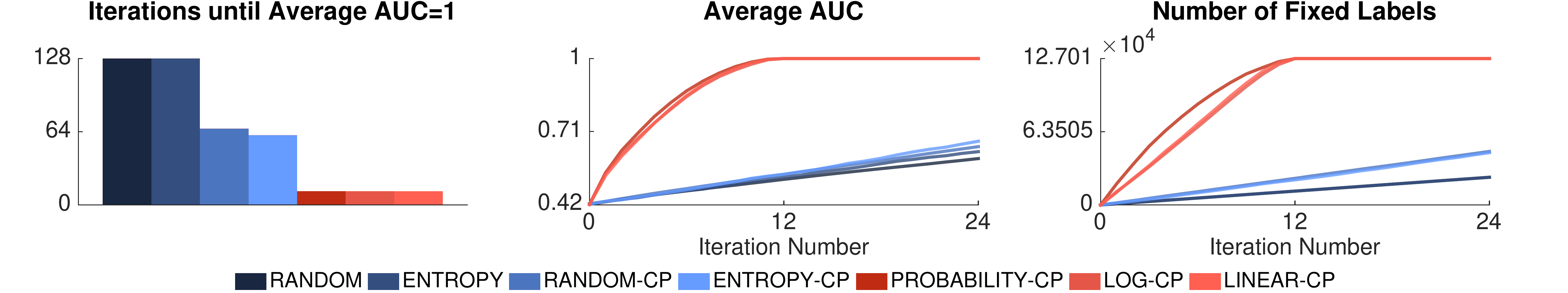} \\
    {\sc Nell-7} & \includegraphics[width=0.92\textwidth,trim=0 60 0 28,clip]{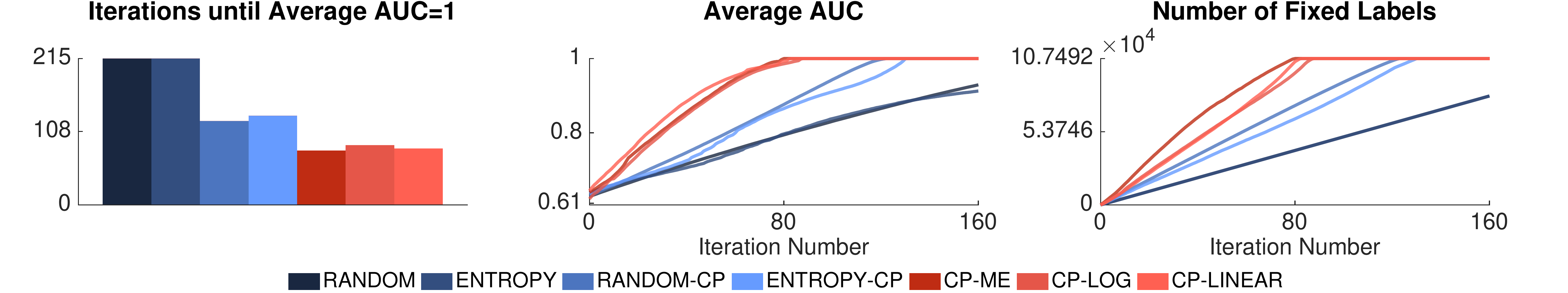} \\
    {\sc Nell-11} & \includegraphics[width=0.92\textwidth,trim=0 60 0 28,clip]{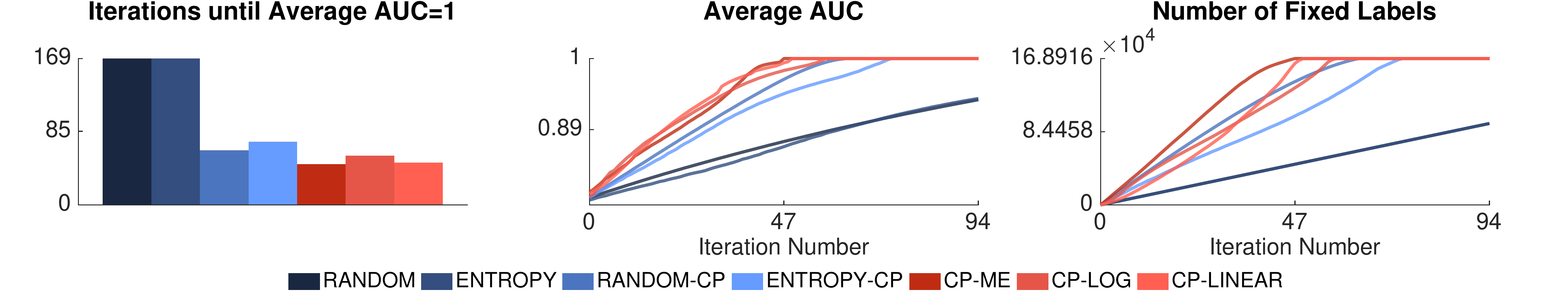} \\
    {\sc Nell-13} & \includegraphics[width=0.92\textwidth,trim=0 0 0 28,clip]{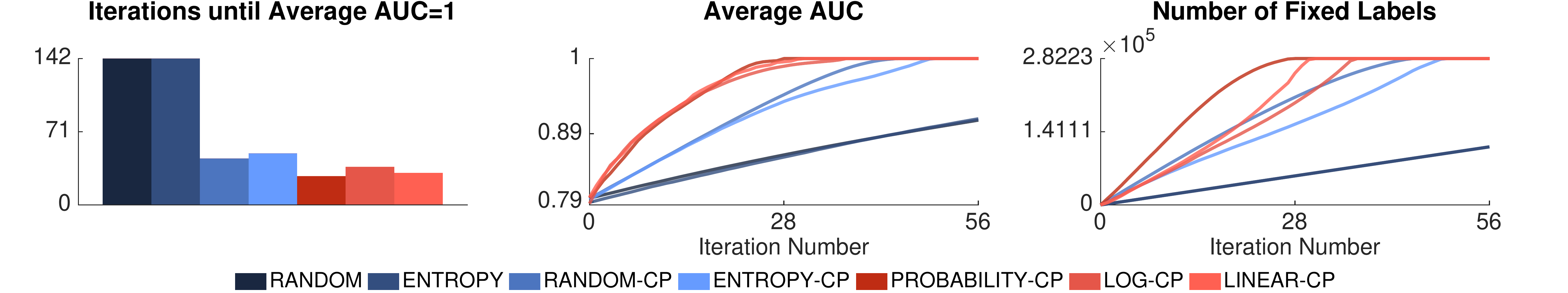}
    \end{tabular}
    \vspace{-0.5em}
    \caption{Results. The red colored plots refer to the proposed methods and the blue colored plots refer to existing methods (apart from the constraint propagation step that we optionally added to all existing methods to enable a fairer comparison, and which is denoted with a \texttt{-CP} appended to the method name). For the first plot, the lower the bar, the better the result. For the rest of the plots, the higher the value of the curve per iteration, the better the result. We thus observe that the proposed methods outperform all existing methods for all of the experiments performed.}
    \label{fig:results}
\end{figure*}

\vspace{-0.4em}
\subsection{Datasets}
\label{sec:data_sets}
\vspace{-0.1em}

We now provide the list of data sets we used for our experiments, with a small description for each data set. Table \ref{tab:data_sets} provides details on the statistics and experimental setup for each dataset. All data sets, except for the NELL data, were obtained from \url{https://www.csie.ntu.edu.tw/~cjlin/libsvmtools/datasets/multiclass.html}. The NELL data sets were obtained from \url{https://rtw.ml.cmu.edu/rtw/resources}. Details on the statistics and experimental setup for each data set are provided in table \ref{tab:data_sets}. Note that for all data sets except {\sc Nell-11} and {\sc Nell-13}, the only constraint used is a mutual exclusion constraint between all labels. The constraints used for those two NELL data sets are detailed in the following list.
\begin{itemize}[noitemsep,topsep=0pt]
	\item {\sc SatImage:} Classify a satellite image region \citep{Feng:1993}.
	\item {\sc Shuttle:} Classify a space shuttle as belonging to one of seven classes \citep{Feng:1993}.
	\item {\sc Segment:} Classify a small outdoor image region \citep{Feng:1993}.
	\item {\sc PenDigits:} Classify a handwritten digit \citep{Alimoglu:1991}.
	\item {\sc Letter:} Classify an image as a letter of the English alphabet \citep{Fray:1991}.
	\item {\sc Nell-7:} Classify noun phrases as belonging to a certain category or not. The categories considered for this data set are \texttt{Bird}, \texttt{Fish}, \texttt{Mammal}, \texttt{City}, \texttt{Country}, \texttt{Lake}, and \texttt{River} (i.e., the category represents the label in this case). The only constraint considered in this case is that all these categories are mutually exclusive. We use the same set of features as that used by the coupled pattern learner (CPL) in NELL \citep{Mitchell:2015wo}.
	\item {\sc Nell-11:} Perform the same task as {\sc Nell-7}, but additionally consider the categories \texttt{Animal}, \texttt{Location}, \texttt{Artificial Location}, and \texttt{Natural Location}. Also include the subsumption constraints shown in figure \ref{fig:NELL_13_constraints}, while ignoring the categories not included in this data set.
	\item {\sc Nell-13:} Perform the same task as {\sc Nell-7}, but with the categories and constraints illustrated in figure \ref{fig:NELL_13_constraints}.
\end{itemize}

\vspace{-0.4em}
\subsection{Evaluation Metrics}
\vspace{-0.1em}

We first define {\em average area under the curve (average AUC)}. At each iteration and for each label, we compute the AUC over the whole dataset (i.e., training dataset and testing dataset combined). Then, we compute a weighted average of the AUCs for each label, where each label's contribution is weighted by the number of positive examples that exist in the dataset, for that label. That weighted average is what we refer to as average AUC. It is easy to see that average AUC is a non-decreasing function with respect to iteration number\footnote{That nice property is the reason we use the combined dataset as opposed to just using the testing dataset.}. We use the following three metrics to evaluate the proposed methods:
\begin{itemize}[noitemsep,topsep=0pt,leftmargin=*]
	\item \uline{Iterations until Average AUC=1:} Number of iterations until average AUC $\geq 0.999$.
	\item \uline{Average AUC:} Average AUC vs iteration number.
	\item \uline{Number of Fixed Labels:} Number of labels that are effectively fixed (i.e., added to the training dataset), after each iteration. Note that this measure is not always equal to the number of labels requested because of the constraint propagation step.
\end{itemize}

\vspace{-0.4em}
\subsection{Results Analysis}
\vspace{-0.1em}

All results are shown in Figure \ref{fig:results}. We first note that the proposed methods {\em consistently beat the other methods by a significant margin, for all datasets and all evaluation metrics}. For the datasets that only consider a single mutual exclusion constraint, \texttt{PROBABILITY-CP} always performs best with respect to the number of iterations until average AUC=1. This is not unexpected; as we showed in Section \ref{sec:proposed_methods}, this method can be considered optimal. Furthermore, we find it interesting that, for the average AUC plots, in all cases where we only have a single mutual exclusion constraint, despite seeing underperformance in early iterations, \texttt{PROBABILITY-CP} still reaches AUC $=1$ faster. This may be based on the fact that this method first selects label-instance pairs with probability very close to $1$, which turn out to be positive. However, after a few iterations, the method experiences a boost and beats all of the other methods. As for the number of fixed labels per iteration, \texttt{PROBABILITY-CP} also beats the other methods by far. This provides validation of the intuition discussed in Section \ref{sec:proposed_methods}, that the method would fix more labels when the mutual exclusion constraint is propagated. As for the two datasets where we also have subsumption constraints and multiple mutual exclusion constraints, we see that the proposed methods consistently outperform all other methods, as expected. We did not expect, however, for \texttt{PROBABILITY-CP} to be doing as well as \texttt{LOG-CP} and \texttt{LINEAR-CP}. We do not yet have an understanding about this finding, but we find this interesting and encouraging for the proposed methods. We note that there are a few datasets with only a single mutual exclusion constraint, where \texttt{PROBABILITY-CP} is actually beaten early on in the average AUC curve, by our other two proposed methods. Thus, there is value in using these two methods in some scenarios. Finally, we found interesting that the constraint propagation step alone provides a significant performance boost to all methods.

\vspace{-0.5em}
\section{Conclusion}
\vspace{-0.1em}

We have proposed methods for performing active learning efficiently in the presence of logical constraints between the outputs of multiple classifiers. The approach resonates with underlying intuitions and challenges the core idea behind uncertainty guided sampling. We provided theoretical justification for using the proposed methods.  In a set of experiments, we found that the methods consistently outperformed competing methods across ten diverse datasets and thus appear to be promising for practical applications. Moreover, the experiments showed that our methods can be used to speed up the learning process in NELL.  Per our knowledge, this paper is the first to describe and carefully study methods for performing active learning when there are logical constraints among outputs of multiple classifiers. 

We are excited about numerous future directions for this work. Our first priority is to pursue additional theoretical results for the general setting with arbitrary logical constraints. We would also like to explore methods for a setting where all labels for a particular data instance are requested at each iteration; this use case is useful to systems like NELL where the label space is extremely sparse. We would also like to explore ways in which we can use accuracy estimates for the trained classifiers (using methods such as those proposed in \citep{Platanios:2016} and \citep{Platanios:2017}, which uses similar logical constraints, for example) in order to make the active learning procedures more robust. Implementing efficient computation of the value of information for multiple, interdependent classifiers would be a step towards autonomous learning systems with the ability to reflect more deeply about their pursuit of information.

\vspace{-0.4em}
\section*{Acknowledgements}
\vspace{-0.1em}

We would like to thank Abulhair Saparov and Otilia Stretcu for the useful feedback they provided in early versions of this paper. This research was performed during an internship at Microsoft Research.


\bibliography{paper}
\bibliographystyle{abbrvnat}

\end{document}


\maketitle

\section{Derivations}

\subsection{Proofs}

\begin{lemma}
\label{lem:mutual_exclusion_theorem_lemma}
Let $x\in[0,1]$, and $c\in[0,1-x]$. Then, the following function is monotonic with respect to $x$: $f(x)=(1-x-c)\log{(1-x-c)}-(1-x)\log{(1-x)}$.
\end{lemma}
\begin{proof}
We have that $\smash{\partial f(x)/\partial x=\log{(1-x)}-\log{(1-x-c)}}$ and since the logarithm is a monotonic function, we know that $\smash{\partial f(x)/\partial x\geq 0}$. Thus, $f(x)$ is monotonic.
\end{proof}

\begin{proposition}
\label{prop:entropy_cp_equivalence_condition}
When:
\begin{equation*}
	\argmax_{\substack{k=1,\hdots,K,\\i=1,\hdots,N}}{p_k^i}=\argmin_{\substack{k=1,\hdots,K,\\i=1,\hdots,N}}{|p_k^i-0.5|},
\end{equation*}
the probability scoring function of Equation 2 is equivalent to the entropy scoring function of Equation 1, which is used by uncertainty guided sampling.
\end{proposition}
\begin{proof}
The proof follows immediately by noticing that for $p_k^i\in[0,1]$, the following holds:
\begin{equation*}
	\argmax_{\substack{k=1,\hdots,K,\\i=1,\hdots,N}}{\mathcal{S}_{\textrm{entropy}}(Y_k^i)}\equiv\argmin_{\substack{k=1,\hdots,K,\\i=1,\hdots,N}}{|p_k^i-0.5|}.\qedhere
\end{equation*}
\end{proof}

\begin{corollary}
\label{cor:entropy_cp_single_instance_special_case}
In the case of a single instance $X$ the probability scoring function of Equation 2 and the entropy scoring function of Equation 1 are equivalent.
\end{corollary}
\begin{proof}
Due to the mutual exclusion constraint, $\smash{\sum_{k=1}^K{p_k}\leq 1}$, which implies that the condition of Proposition \ref{prop:entropy_cp_equivalence_condition} is always satisfied.
\end{proof}

\subsection{Derivation for Section 2.2.1}

We start with the definition of information gain. Note that we shall ignore terms that are constant across all label variables, since these terms do not affect the ranking of the labels induced by information gain.
\begin{align*}
	&\mathcal{I}(Y_k)=\mathcal{H}(Y_{-k})-\mathcal{H}(Y_{-k}\mid\mathcal{H}_k)=\mathcal{H}(Y_{-k})+\mathcal{H}(Y_k)-H(Y)=\mathcal{H}(Y_k)-\mathcal{H}(Y_k\mid\bm{Y_{-k}}), \\
	&\;=\sum_{y_k}\bigg[-\mathbb{P}(y_k)\log{\mathbb{P}(y_k)}+\sum_{\bm{y_{-k}}}{\mathbb{P}(\bm{y_{-k}})\mathbb{P}(y_k\mid\bm{y_{-k}})\log{\mathbb{P}(y_k\mid\bm{y_{-k}})}}\bigg], \\
	&\;=\sum_{y_k}\bigg[-\mathbb{P}(y_k)\log{\mathbb{P}(y_k)}+\sum_{\bm{y_{-k}}}\big[\underbrace{\mathbb{P}(y_k,\bm{y_{-k}})\log{\mathbb{P}(y_k,\bm{y_{-k}})}}_{\textrm{Constant}}-\mathbb{P}(y_k,\bm{y_{-k}})\log{\mathbb{P}(\bm{y_{-k}})\big]}\bigg], \\
	&\;=\sum_{y_k}\mathbb{P}(y_k)\bigg[-\log{\mathbb{P}(y_k)}-\sum_{\mathclap{\substack{\bm{y_{-f}}\\\textrm{where }\bm{y_f}=\mathcal{F}(y_k)}}}{\quad\underbrace{\mathbb{P}(\bm{y_{-f}},\bm{y}_{\bm{f} \setminus k}\mid y_k)}_{\mathbb{P}(\bm{y_{-f}}\mid\bm{y_f})}\log{\mathbb{P}(\bm{y_{-f}},\bm{y}_{\bm{f} \setminus k})}}\bigg], \\
	&\;=\sum_{y_k}\mathbb{P}(y_k)\bigg[-\log{\mathbb{P}(y_k)}-\sum_{\mathclap{\substack{\bm{y_{-f}}\\\textrm{where }\bm{y_f}=\mathcal{F}(y_k)}}}\quad\big[\underbrace{\mathbb{P}(\bm{y_{-f}}\mid\bm{y_f})}_{\textrm{Sums to }1}\log{\mathbb{P}(\bm{y_f})}+\mathbb{P}(\bm{y_{-f}}\mid\bm{y_f})\log{\mathbb{P}(\bm{y_{-f}}\mid\bm{y}_{\bm{f} \setminus k})}\big]\bigg], \\
	&\;=\sum_{y_k}\mathbb{P}(y_k)\bigg[-\underbrace{\log{\mathbb{P}(y_k)}}_{\textrm{Entropy}}-\underbrace{\log{\mathbb{P}(\bm{y_f})}}_{\textrm{Constraints}}\quad-\quad\sum_{\mathclap{\substack{\bm{y_{-f}}\\\textrm{where }\bm{y_f}=\mathcal{F}(y_k)}}}{\quad\underbrace{\mathbb{P}(\bm{y_{-f}}\mid\bm{y_f})\log{\mathbb{P}(\bm{y_{-f}}\mid\bm{y}_{\bm{f} \setminus k})}}_{\textrm{Remainder}}}\bigg],
\end{align*}
where $\bm{f}\setminus k$ is the set of label indices in $\bm{f}$ excluding $k$.

\section{Experiments}

\subsection{Datasets}
\label{sec:data_sets}

In this section we provide the list of datasets we used for our experiments, with a small description for each dataset. All datasets, except for the NELL data, were obtained from \url{https://www.csie.ntu.edu.tw/~cjlin/libsvmtools/datasets/multiclass.html}. The NELL datasets were obtained from \url{https://rtw.ml.cmu.edu/rtw/resources}. Details on the statistics and experimental setup for each dataset are provided in Table 1 of the main paper. Note that for all datasets except {\sc Nell-11} and {\sc Nell-13}, the only constraint used is a mutual exclusion constraint between all labels. The constraints used for those two NELL datasets are detailed in the following list.
\begin{itemize}[noitemsep,topsep=0pt,leftmargin=*]
	\item {\sc Iris:} Classify an iris plant \citep{Fisher:1936}.
	\item {\sc SatImage:} Classify a satellite image region \citep{Feng:1993}.
	\item {\sc Shuttle:} Classify a space shuttle as belonging to one of seven classes \citep{Feng:1993}.
	\item {\sc Segment:} Classify a small outdoor image region \citep{Feng:1993}.
	\item {\sc PenDigits:} Classify a handwritten digit \citep{Alimoglu:1991}.
	\item {\sc Vowel:} Classify an utterance of a vowel \citep{Deterding:89}.
	\item {\sc Letter:} Classify an image as a letter of the English alphabet \citep{Fray:1991}.
	\item {\sc Nell-7:} Classify noun phrases as belonging to a certain category or not. The categories considered for this dataset are \texttt{Bird}, \texttt{Fish}, \texttt{Mammal}, \texttt{City}, \texttt{Country}, \texttt{Lake}, and \texttt{River} (i.e., the category represents the label in this case). The only constraint considered in this case is that all these categories are mutually exclusive. We use the same set of features as that used by the coupled pattern learner (CPL) in NELL \citep{Mitchell:2015wo}.
	\item {\sc Nell-11:} Perform the same task as {\sc Nell-7}, but additionally consider the categories \texttt{Animal}, \texttt{Location}, \texttt{Artificial Location}, and \texttt{Natural Location}. Also include the subsumption constraints shown in Figure 2 of the main paper, while ignoring the categories not included in this dataset.
	\item {\sc Nell-13:} Perform the same task as {\sc Nell-7}, but with the categories and constraints illustrated in Figure 2 of the main paper.
\end{itemize}

\bibliography{paper}
\bibliographystyle{abbrvnat}